\documentclass[conference]{IEEEtran}

\usepackage[utf8]{inputenc} 
\usepackage[T1]{fontenc}  
\usepackage{cite}             

\usepackage[cmex10]{amsmath}  
\usepackage{mleftright}       
\mleftright                   

\usepackage{graphicx}         
\usepackage{booktabs}         

\usepackage{mathtools,amssymb}
\usepackage{algorithm,algorithmic}
\graphicspath{{images/}}
\allowdisplaybreaks
\usepackage{enumitem}

\ifCLASSINFOpdf

\else

\fi

\newtheorem{theorem}{Theorem}
\newtheorem{lemma}{Lemma}

\newtheorem{definition}[theorem]{Definition}

\usepackage{balance}
\begin{document}

\title{Best Arm Identification in Bandits with \\
Limited Precision Sampling}

\author{%
  \IEEEauthorblockN{Kota Srinivas Reddy\IEEEauthorrefmark{1}, P. N. Karthik\IEEEauthorrefmark{2}, Nikhil Karamchandani\IEEEauthorrefmark{3}, and Jayakrishnan Nair\IEEEauthorrefmark{3}}
  \IEEEauthorblockA{\IEEEauthorrefmark{1}Indian Institute of Technology Madras, \IEEEauthorrefmark{2} National University of Singapore, \IEEEauthorrefmark{3} Indian Institute of Technology Bombay \\
Emails: ksvr1532@gmail.com, karthik@nus.edu.sg, nikhilk@ee.iitb.ac.in, jayakrishnan.nair@ee.iitb.ac.in}
}

\maketitle

\begin{abstract}
    We study best arm identification in a variant of the multi-armed bandit problem where the learner has limited precision in arm selection. The learner can only sample arms via certain exploration bundles, which we refer to as boxes. In particular, at each sampling epoch, the learner selects a box, which in turn causes an arm to get pulled as per a box-specific probability distribution. The pulled arm and its instantaneous reward are revealed to the learner, whose goal is to find the best arm by minimising the expected stopping time, subject to an upper bound on the error probability. We present an asymptotic lower bound on the expected stopping time, which holds as the error probability vanishes. We show that the optimal allocation suggested by the lower bound is, in general, non-unique and therefore challenging to track. We propose a modified tracking-based algorithm to handle non-unique optimal allocations, and demonstrate that it is asymptotically optimal. We also present non-asymptotic lower and upper bounds on the stopping time in the simpler setting when the arms accessible from one box do not overlap with those of others.

\end{abstract}

\section{Introduction}
\label{sec:introduction}

In this paper, we study best arm identification in a multi-armed bandit setting with $K$ arms, where the learner has limited precision in sampling arms. In particular, the learner cannot directly sample individual arms, but can instead sample only certain exploration bundles, which we refer to as {\em boxes}. Each box is associated with a probability distribution over the arms;
upon selecting a box, an arm is pulled randomly according to its corresponding probability distribution. The learner sees the pulled arm and its instantaneous reward. The learner's goal is to find the {\em best arm}, defined as the arm with the largest mean reward, while  minimising the expected stopping time subject to an upper bound on the error probability (i.e., {\em fixed-confidence} regime). For the {\em boxed-bandit best arm identification} problem described above, our objective is to design sound algorithms and benchmark their performance against information theoretic lower bounds, when the arm reward distributions and the arm selection probabilities of the various boxes are apriori unknown.

\subsection{Motivation}

The key feature of our model is that the learner does not have direct access to the arms. Instead, it must perform its exploration via certain intermediaries (boxes), which have their own preferences/biases/constraints over arm selection. To consider a contemporary example, suppose the goal of the learner is to identify the most contagious strain of a virus/pathogen in a large community (say a country), by ordering tests at different local testing facilities. Each testing facility in turn performs its tests by sampling individuals in its local vicinity/jurisdiction; the likelihood of encountering different strains being a function of the facility location.

Another interpretation of our model is that it captures \emph{noise} in arm selection. For example, when the learner attempts to pull a certain arm, the pull is only executed successfully with, say probability~$1-\eta;$ with probability~$\eta,$ either no arm is pulled (i.e., an \emph{erasure} occurs), or a random arm is pulled (this is the \emph{trembling hand} model of \cite{karthik2021detecting}).

Finally, our model can also be interpreted as a \emph{privacy preservation} exercise on the part of learner. By performing its exploration via the \emph{non-adaptive} selection profiles of the boxes, the learner can obfuscate its own preferences from other observers. Naturally, this obfuscation comes at the expense of increased sampling complexity. Under this alternative interpretation, it may be reasonable to assume that the learner knows the arm selection probabilities of the boxes; our algorithms simplify naturally to this special case.


\subsection{Analytical Challenges}
Notice that in our problem setup, the learner only has {\em partial} control over the arms (via the boxes). This is unlike classical best arm identification problems \cite{kalyanakrishnan2012pac,kaufmann2016complexity,garivier2016optimal} where the learner has full control over the arm to pull at each time instant. For instance, in the Successive Elimination algorithm of \cite{even2006action} or the LUCB algorithm of \cite{jamieson2014best}, the learner pulls one or more arms at each time instant and either eliminates the sub-optimal arms or resolves between the best and second-best arms on-the-fly to eventually arrive at the best arm. In our setup, because a given arm may be accessible via multiple boxes, and the arm selection probabilities of the boxes are not known beforehand, it is not clear at the outset which box must be selected more frequently to maximise the chances of pulling an arm. In fact, if each arm belongs to every box and the arms selection probabilities of the boxes are all identical, then every randomised box selection rule yields the same expected stopping time. Thus, elimination or LUCB-type algorithms do not apply verbatim to our setting. 

We also note that the lower bounds appearing in \cite{garivier2016optimal,kaufmann2016complexity} admit a unique optimal solution (or {\em allocation}), and a key aspect of the best arm identification algorithms in these works is {\em tracking} or the convergence of the empirical arm selection frequencies to the optimal allocation. In contrast, we show that the optimal allocation in our setup is in general non-unique. In this case, the empirical frequencies may alternate between two or more optimal allocations and not converge to any of the optimal allocations in the long run. This underscores the need to improvise the existing tracking-based algorithms of~\cite{garivier2016optimal,kaufmann2016complexity} to handle non-unique optimal allocations.

\subsection{Contributions}
We derive an asymptotic lower bound on the growth rate of the expected stopping time, where the asymptotics is as the error probability vanishes. We show that this growth rate is captured by a sup-inf optimisation problem whose optimal (sup-attaining) solution (or {\em allocation}) is potentially non-unique. Inspired from the analytical techniques of Jedra et al. \cite{jedra2020optimal}, we propose a tracking-based algorithm that is improvised to handle non-unique allocations at every time step, and demonstrate that our algorithm achieves the lower bound asymptotically. In our achievability analysis, we track the long-term behaviour of the empirical average of all the past allocations, and show that the mean allocation eventually approaches the correct set of optimal allocations. Our achievability analysis can potentially be applied to more general problem settings where non-unique allocations arise naturally as in our work, or where proving the uniqueness of the optimal allocation is hard; see, for instance, the remarks in \cite{karthik2021detecting,chen2022federated}. 

Finally, in the special case when the arms are {\em partitioned} among the boxes, i.e., when set of arms accessible from one box does not overlap with that of the others, we present non-asymptotic guarantees for a variant of the successive elimination algorithm. We show that the expected stopping time of this algorithm satisfies an upper bound that is tight in the unknown instance parameters.



\subsection{Related Works}
Our setup closely resembles that in~\cite{scarlett2019overlapping}, where the arms are grouped into subsets as in our work, with potential overlap between the sets. However, the key difference is that in  \cite{scarlett2019overlapping}, the learner has full control over the arms (unlike partial control of arms in our work). Also, in \cite{scarlett2019overlapping}, the goal is to find the best arm within each subset, whereas in our work, the goal is to find the overall best arm. The paper~\cite{jain2021sequential} considers a similar setup as ours for a problem of community mode estimation, with the key difference that the analysis and results in \cite{jain2021sequential} are for the {\em fixed-budget} regime, whereas those of our work are for the {\em fixed-confidence} regime; see \cite{kaufmann2016complexity} for a comparison of these regimes. Our setup specialises to those in \cite{kaufmann2016complexity,even2006action,jamieson2014best,kalyanakrishnan2012pac} when the number of boxes equals the number of arms, and each box contains one arm. Our setup also specialises to the {\em trembling hand model} of \cite{karthik2021detecting}; in this model, when the learner attempts to pull arm~$k,$ it actually gets pulled with probability~$1-\eta,$ whereas a random arm, chosen uniformly, gets pulled with probability~$\eta.$
\section{Formulation and Preliminaries}
\label{sec:problem-setting}

We consider a $K$-armed bandit, where the arms are labelled $1, 2,
\ldots,K$. Arm~$k \in [K]$\footnote{Let $[n]\coloneqq \{1,
  \ldots, n\}$ for any integer $n \geq 1$.} is associated with a reward
distribution~$\nu_k \in \mathcal{G}$, where $\mathcal{G}$ is a known class of arm distributions. Let~$\mu_k$ denote the
mean reward of arm~$k.$ The goal of the learner is to identify, via
sequential sampling, the optimal arm, which is defined to be the arm
having the largest mean reward. However, unlike in the classical MAB
setting, the learner cannot sample (a.k.a., pull) individual arms
directly. Instead, at each epoch, the learner selects a \emph{box}
(from a finite collection of $M$ boxes), which results in an arm being
pulled randomly according to a box-dependent probability
distribution. Formally, selecting box~$m \in [M]$ results in arm~$k \in \mathcal{A}_m \subseteq [K]$ being
pulled with probability~$q_{m,k}.$ Here, $\mathcal{A}_m$ denotes the set of arms
that are `accessible' using box~$m$ (i.e., $q_{m,k} > 0$ for $k \in \mathcal{A}_m$
and $\sum_{k \in \mathcal{A}_m} q_{m,k} = 1$). Importantly, $\boldsymbol{q}
\coloneqq \{q_{m,k}: k \in \mathcal{A}_m, m \in [M]\}$ is apriori unknown to the
learner. Note that~$\boldsymbol{q}$ describes the imprecision in the learner's ability to pull specific arms. Indeed, capturing the impact of this sampling imprecision on the complexity of best arm identification is the main focus of this work.

The tuple $C=(\boldsymbol{q}, \boldsymbol{\nu})$ completely specifies
a {\em problem instance}, where $\boldsymbol{\nu} = (\nu_k: k \in
[K])$ is the vector of arm distributions. The optimal arm
corresponding to this problem instance is denoted by $a^\star(C) =
a^\star(\boldsymbol{\mu}) = \arg\max_{k \in [K]} \mu_k,$ where
$\boldsymbol{\mu} = (\mu_k : k \in [K]).$ The best arm is assumed to be uniquely defined for every problem instance. We write $\textsc{Alt}(\boldsymbol{\mu})$ to denote the set of instances {\em alternative} to $\boldsymbol{\mu}$, i.e., those instances whose best arm differs from $a^\star(\boldsymbol{\mu})$. When there is no ambiguity, we write $C=(\boldsymbol{q}, \boldsymbol{\mu})$ in place of $C=(\boldsymbol{q}, \boldsymbol{\nu})$. 

For $t \in \{1, 2, \ldots\}$, let $B_t$ denote the box selected by the
learner at time $t$. Upon selecting box $B_t=m$, arm $A_t=k$ is
pulled with probability $q_{m,k}$. The learner observes $A_t$ (i.e.,
it knows which arm was pulled) and the reward $X_t$ from arm
$A_t$. Let $(B_{1:t}, A_{1:t}, X_{1:t}) \coloneqq (B_1, A_1, X_1,
\ldots, B_{t}, A_{t}, X_{t})$ denote the history of box selections,
arm pulls, and observations seen up to time $t$. Given $\delta \in
(0,1)$, the goal of the learner is to find the best arm with the least
expected number of box selections (a.k.a. expected stopping time), while keeping the stoppage error
probability below $\delta$.

Let $\pi=\{\pi_t\}_{t=1}^{\infty}$ denote any generic best arm
identification {\em policy} (or algorithm), where for every $t \geq
1$, $\pi_t$ maps the history $(B_{1:t}, A_{1:t}, X_{1:t})$ to one of
the following actions:
\begin{itemize}
    \item Select box $B_{t+1}$ according to a deterministic or randomised rule.
    \item Stop and declare the estimated best arm.
\end{itemize}
Let $\tau_\pi$ denote the (random) stopping time under $\pi$, and let
$\hat{a}$ be the best arm estimate at stoppage.  For each $\delta\in
(0,1)$, our interest is in the class of {\em $\delta$-probably
  correct} ($\delta$-PC) algorithms defined by $\Pi(\delta) \coloneqq
\{\pi: P(\hat{a} \neq a^\star(C)) \leq \delta\quad \forall C\}.$ In
this paper, we design $\delta$-PC algorithms and
benchmark their stopping times against information theoretic
lower bounds.

In Section~\ref{sec:np}, we design and analyse a track-and-stop style
algorithm taking the class of arm distributions $\mathcal{G}$ to be
the family of Gaussian distributions with a known variance. This
algorithm is shown to be $\delta$-PC, and its expected stopping time is
shown to be asymptotically optimal as $\delta \downarrow 0.$ Next, in
Section~\ref{sec:partition-setting}, we consider the special case of our model where the
sets~$\{\mathcal{A}_m: m\in [M]\}$ are disjoint (i.e., the arms are {\em partitioned}
across boxes). For this case, taking $\mathcal{G}$ to be the family of
1-sub-Gaussian distributions, we design and analyse a successive-elimination style
algorithm which admits non-asymptotic (in $\delta$) stopping time
bounds.
 

\section{Track \& Stop Based Algorithm}
\label{sec:np}

We first study the general setting when each arm may be associated with multiple boxes. 
For simplicity in presentation, we assume that the observations from arm $k$ are Gaussian distributed with mean $\mu_k$ and variance $1$. 
Without loss of generality, we present our results for the extreme setting when each arm is associated with {\em every} box. 
Let the underlying instance be defined by $\boldsymbol{q}_0=\{q_{m,k}^{0}: m \in [M], k \in [K]\}$ and $\boldsymbol{\mu}_0=\{\mu_k^{0}: k \in [K]\}$. We first present an asymptotic (as $\delta \downarrow 0$) lower bound on the growth rate of the expected stopping time. 
Following this, we highlight the central challenge in the analysis of the non-partition setting: {\em non-uniqueness} of the optimal solution to the optimization problem that characterizes the lower bound. We then present a new track-and-stop based algorithm inspired from \cite{jedra2020optimal} and demonstrate its asymptotic optimality despite the above challenge. 



\subsection{Converse: Asymptotic Lower Bound}
The first main result of this section, a lower bound on the limiting growth rate of the expected stopping time for $\delta$-PC algorithms in the limit as $\delta \downarrow 0$, is presented below. 
\begin{theorem}
\label{prop:lower-bound}
Let $\boldsymbol{q}_0=\{q_{m,k}^{0}\}_{m,k}$, $\boldsymbol{\mu}_0=\{\mu_k^0\}_{k=1}^{K}$. Then,
\begin{equation}
    \liminf_{\delta \downarrow 0} \ \inf_{\pi \in \Pi(\delta)}\ \frac{\mathbb{E}[\tau_\pi]}{\log(1/\delta)} \geq \frac{1}{T^*(\boldsymbol{q}_0, \boldsymbol{\mu}_0)},
    \label{eq:lower-bound}
\end{equation}
where $T^*(\boldsymbol{q}_0, \boldsymbol{\mu}_0)$ in \eqref{eq:lower-bound} is given by
\begin{equation}
    T^*(\boldsymbol{q}_0, \boldsymbol{\mu}_0) \! = \!\sup_{w\in \Sigma_M} \inf_{\boldsymbol{\lambda} \in \textsc{Alt}(\boldsymbol{\mu}_0)} \sum_{m=1}^{M} \sum_{k=1}^{K} w_m \,  q_{m,k}^{0}\, \frac{(\mu_{k}^{0}- \lambda_{k})^2}{2},
    \label{eq:T-star-C}
\end{equation}
where $\Sigma_M$ is the simplex of all probability distributions (or {\em allocations}) $w=(w_1, \ldots, w_M)$ on the boxes.

\end{theorem}
The proof of Theorem~\ref{prop:lower-bound} is quite standard and omitted for brevity. It employs a change-of-measure argument for bandits~\cite{lairobbins1985}, the transportation lemma of \cite{kaufmann2016complexity}, and Wald's identity. 

\subsection{Non-Uniqueness of the Optimal Allocation}
Consider the following simple example with $M=2$ boxes and $K=4$ arms. Suppose that $\boldsymbol{\mu}_0=\{0.5, 0.4, 0.3, 0.3\}$. Notice that arm $1$ is the best arm. Let $\boldsymbol{q}_0$ be specified by the following matrix:
$$
\boldsymbol{q}_0 =
\begin{pmatrix}
$0.3$ & $0.3$ & $0.3$ & $0.1$\\
$0.3$ & $0.3$ & $0.1$ & $0.3$
\end{pmatrix}.
$$
The first row of the above matrix represents the arm selection probabilities of box $1$, and the second row that of box $2$. For this example, it is easy to show that every $w=(w_1, w_2) \in \Sigma_2$ attains the supremum in \eqref{eq:T-star-C} and hence is an optimal allocation. 

The above example shows that the optimal allocation can potentially be non-unique. This is in contrast to the prior works \cite{garivier2016optimal,kaufmann2016complexity} where the optimal allocation is unique. Let
$\mathcal{W}^\star(\boldsymbol{q}_0, \boldsymbol{\mu}_0)$ denote the set of all allocations that attain the supremum in \eqref{eq:T-star-C}. More generally, let $\mathcal{W}^\star(\boldsymbol{q}, \boldsymbol{\mu})$ denote the set of optimal allocations corresponding to the instance $(\boldsymbol{q}, \boldsymbol{\mu})$.
\begin{lemma}
\label{lemma:berges-plus-convexity}
    The mapping $(\boldsymbol{q}, \boldsymbol{\mu}) \mapsto \mathcal{W}^\star(\boldsymbol{q}, \boldsymbol{\mu})$ is upper-hemicontinuous and compact-valued. Furthermore, $\mathcal{W}^\star(\boldsymbol{q}, \boldsymbol{\mu})$ is convex for all $(\boldsymbol{q}, \boldsymbol{\mu})$.
\end{lemma}
In particular, Lemma~\ref{lemma:berges-plus-convexity} implies that $\mathcal{W}^\star(\boldsymbol{q}_0, \boldsymbol{\mu}_0)$ is convex; this, we shall see, will play an important role in the design of an asymptotically optimal algorithm, which forms the content of the next section.

\subsection{Achievability: Handling Non-Unique Allocations}
A key feature of the best arm identification algorithms and the contingent achievability analyses in the prior works \cite{kaufmann2016complexity,garivier2016optimal} is {\em tracking}, or the almost sure convergence of the empirical frequencies of arm pulls to the optimal allocation. When the optimal allocation is non-unique as in our work, the empirical frequencies may alternate among two or more optimal allocations in the long-run and not converge to any one of the optimal allocations, in which case it is difficult to establish tracking. This emphasises the need to improvise the existing tracking-based algorithms to handle potentially non-unique optimal allocations and develop a framework to demonstrate tracking-like behaviour. In this section, we present a technique 
that achieves this.

Let $N(t,m,k)$ be the total number of times box $m$ is selected and arm $k$ is pulled up to time $t$. 
Define $N(t,m)=\sum_k N(t,m,k)$ be the number of times box $m$ is selected up to time $t$, and $N_k(t)=\sum_m N(t,m,k)$ be the number of times arm $k$ is pulled up to time $t$. For all $m,k$, let 
\begin{equation}
    \hat{q}_{m,k}(t)=\frac{N(t,m,k)}{N(t,m)}, \quad \hat{\mu}_{k}(t)=\frac{1}{N_k(t)} \sum_{s=1}^{t} \mathbf{1}_{\{A_s=k\}}\, X_s,
    \label{eq:parameter-estimates}
\end{equation}
be the empirical estimates of the unknown parameters at time~$t$.
The key result (inspired from \cite{jedra2020optimal}) that enables us to prove achievability while working with a set of optimal allocations, is stated below.
\begin{lemma}
\label{lemma:modified-tracking-with-non-unique-optimal-solutions}
Let $f(t)=\frac{\sqrt{t}}{\sqrt{M}}$. Let $\{w(t)\}_{t=1}^{\infty} \subset \Sigma_M$ be any sequence such that $w(t+1) \in \mathcal{W}^\star(\hat{\boldsymbol{q}}(t), \hat{\boldsymbol{\mu}}(t))$ for all $t$. Let $i_0=0$ and
$$
i_{t+1}=(i_{t}\ {\rm mod}\ M)+\mathbf{1}_{\left\lbrace\min_{m\in [M]} N(t, m) < f(t)\right\rbrace}, \quad  t \geq 0.
$$ 
Then, under the {\em modified D-tracking rule} given by
\begin{equation}
    B_{t+1} = 
    \begin{cases}
        i_{t}, & \min_{m\in [M]} N(t, m) < f(t),\\
        b_t, & \text{otherwise},
    \end{cases}
    \label{eq:sampling-rule}
\end{equation}
where $b_t=\arg\min_{m \in \text{supp}(\sum_{s=1}^{t} w(s))} N(t, m) - \sum_{s=1}^{t} w_m(s)$, we have \begin{equation}
    \lim_{t \to\infty} d_{\infty}((N(t, m)/t)_{m \in [M]},\ \mathcal{W}^\star(\boldsymbol{q}_0, \boldsymbol{\mu}_0))=0 \quad \text{a.s.}.
    \label{eq:arm-frequencies-approach-the-correct-set}
\end{equation}
\end{lemma}
Notice that the sampling rule in \eqref{eq:sampling-rule} selects each of the boxes once at the beginning. At time $t$, if any one of the boxes is under-sampled, i.e., $\min_m N(t,m)<f(t)$, the boxes are sampled forcefully in a round-robin fashion until $N(t,m)=\Omega(\sqrt{t})$ for all $m$. 
Else, a box is sampled based the allocations $\{w(s): 1 \leq s \leq t\}.$ Here, $w(s)$ is an \emph{arbitrary} element of $\mathcal{W}^\star(\hat{\boldsymbol{q}}(s-1), \hat{\boldsymbol{\mu}}(s-1)),$ the set of optimal allocations corresponding to the \emph{estimated} instance parameters at time~$s-1.$
This is the principle of {\em certainty equivalence.} In the proof, we show 
that the empirical average of all the allocations up to time $t$, $\bar{w}(t)=1/t \sum_{s=1}^{t} w(s)$, approaches $\mathcal{W}^\star(\boldsymbol{q}_0, \boldsymbol{\mu}_0)$ as $t\to\infty$. Thanks to the forced exploration of the boxes, $(\hat{\boldsymbol{q}}(t), \hat{\boldsymbol{\mu}}(t))$ approaches $(\boldsymbol{q}_0, \boldsymbol{\mu}_0)$ a.s. as $t \to \infty$. This, together with the upper-hemicontinuity property of $\mathcal{W}^\star$ from Lemma~\ref{lemma:berges-plus-convexity}, implies that $\mathcal{W}^\star({\boldsymbol{q}}(t), {\boldsymbol{\mu}}(t))$ is ``close'' to $\mathcal{W}^\star(\boldsymbol{q}_0, \boldsymbol{\mu}_0)$ for large~$t$, a.s., from which we arrive at \eqref{eq:arm-frequencies-approach-the-correct-set}.

Let $Z_{a,b}(t)$ denote the {\em generalised likelihood ratio test statistic} between arms $a,b \in [K]$ up to time $t$, defined as
\begin{equation}
    Z_{a, b}(t) \coloneqq \log \frac{\sup_{(\boldsymbol{q}, \boldsymbol{\mu}): \mu_{a} \geq \mu_b} P(B_{1:t}, A_{1:t}, X_{1:t})}{\sup_{(\boldsymbol{q}, \boldsymbol{\mu}): \mu_{a} \leq \mu_b} P(B_{1:t}, A_{1:t}, X_{1:t})}.
    \label{eq:GLRT-statistic}
\end{equation}
The next result provides an explicit expression for $Z_{a, b}(t)$.
\begin{lemma}
\label{lemma:explicit-expression-for-Z-a-b}
Fix $a, b \in [K]$ and a policy $\pi$. Fix $t$ such that $\min_{k \in [K]} N_k(t)>0$ a.s.. If $\hat{\mu}_a(t) \geq \hat{\mu}_b(t)$, then
\begin{equation}
    Z_{a, b}^\pi(t) 
    \!=\! N_a(t) \frac{\left( \hat{\mu}_a(t) - \hat{\mu}_{a, b}(t) \right)^2}{2} + N_b(t) \frac{\left( \hat{\mu}_b(t) - \hat{\mu}_{a, b}(t) \right)^2}{2},
    \label{eq:test-statistic-2}
\end{equation}
where $\hat{\mu}_{a, b}(t)$ is defined as
\begin{equation}
    \hat{\mu}_{a, b}(t) \coloneqq \frac{N_a(t)}{N_a(t) + N_b(t)} \hat{\mu}_a(t) + \frac{N_b(t)}{N_a(t) + N_b(t)} \hat{\mu}_b(t).
    \label{eq:mu-hat-a-b-of-n}
\end{equation}
If $\hat{\mu}_a(t) \leq \hat{\mu}_b(t)$, then $Z_{a,b}^\pi(t) = -Z_{b,a}^\pi(t)$.
Thus, $Z_{a, b}^\pi(t) \geq 0$ if and only if $\hat{\mu}_a(t) \geq \hat{\mu}_b(t)$.
\end{lemma}

Let $Z(t)=\max_a \min_{b \neq a} Z_{a,b}(t)$. For fixed $\delta \in (0,1)$ and $\rho>0$, let 
$
\zeta(t, \delta, \rho)\coloneqq \log (C \, t^{1+\rho}/\delta)
$, where $C$ is a constant that satisfies $\sum_{t=1}^{\infty} \frac{e^{K+1}}{K^K} \frac{(\log^2(C\, t^{1+\rho}) \, \log t)^K}{t^{1+\rho}} \leq C$.

\begin{algorithm}[!t]
 \caption{Boxed-Bandit Modified Track-and-Stop}
 \begin{algorithmic}[1]
 \renewcommand{\algorithmicrequire}{\textbf{Input:}}
 \renewcommand{\algorithmicensure}{\textbf{Output:}}
 \REQUIRE $\delta \in (0, 1)$, $\rho>0$, $K \in \mathbb{N}$, and $M \in \mathbb{N}$
 \ENSURE $\hat{a} \in [K]$ -- best arm\\
 \textbf{Initialisation: } $t=0$, $\hat{\mu}_a(t)=0\, \forall a \in [K]$, $Z(0)=0$.
 \STATE Compute $\hat{a}(t)=\arg\max_{a} \hat{\mu}_a(t)$.
 \IF{$Z(t) \geq \zeta(t, \delta, \rho)$ and $\min_k N_k(t) > 0$}
    \STATE Stop box selections.
    \RETURN $\hat{a} = \arg\max_{a \in [K]} \hat{\mu}_a(t)$
 \ELSE
    \STATE Select box $B_{t+1}$ as per {\em modified D-tracking rule} \eqref{eq:sampling-rule}.
    \STATE Update $\hat{\boldsymbol{q}}(t)$, $\hat{\boldsymbol{\mu}}(t)$, and $Z(t)$. Go to step 1.
 \ENDIF
 \end{algorithmic}
 \label{algo:algorithm-BBMTS}
\end{algorithm}

\vspace{0.5em}
\noindent {\bf Algorithm for best arm identification:}  We propose a variant of the track-and-stop algorithm of Garivier et al. \cite{garivier2016optimal}, called {\em Boxed-Bandit Modified Track-and-Stop} or BBMTS, that is improvised to work with a set of allocations at each time step. Our algorithm takes as input two parameters: $\delta \in (0,1)$ and $\rho>0$. At each time instant $t$, the algorithm maintains an estimate of the best arm $\hat{a}(t) \in \arg\max_{a} \hat{\mu}_a(t)$, with ties resolved uniformly randomly. Lemma~\ref{lemma:explicit-expression-for-Z-a-b} shows that $\hat{a}(t) \in \arg\max_a \min_{b \neq a} Z_{a,b}(t)$. Then, the algorithm checks if $\min_{b \neq \hat{a}(t)} Z_{\hat{a}(t), b}(t) \geq \zeta(t, \delta, \rho)$. If this holds, the algorithm is sufficiently confident that $\hat{a}(t)$ is the best arm; it stops and outputs $\hat{a}(t)$ as the best arm. Else, the algorithm samples box $B_{t+1}$ according to the modified D-tracking rule \eqref{eq:sampling-rule} as explained earlier. See Algorithm~\ref{algo:algorithm-BBMTS} for a pseudo-code.

\vspace{0.5em}

\noindent {\bf Performance analysis of BBMTS:} Let $\pi_{\text{BBMTS}}(\delta, \rho)$ symbolically represent the BBMTS algorithm with parameters $\delta, \rho$. The below result characterises its performance.
\begin{theorem}
\label{thm:performance-of-BBMTS}
The BBMTS algorithm meets the following performance criteria.
\begin{enumerate}
    \item $\pi_{\text{BBMTS}}(\delta, \rho) \in \Pi(\delta)$ for each $\delta \in (0,1)$ and $\rho > 0$.
    
    \item For each $\rho>0$, the stopping time of $\pi_{\text{BBMTS}}(\delta, \rho)$ satisfies 
    \begin{equation}
        \limsup_{\delta \downarrow 0} \frac{\tau_{\pi_{\text{BBMTS}}(\delta, \rho)}}{\log(1/\delta)} \leq \frac{1+\rho}{T^*(\boldsymbol{q}_0, \boldsymbol{\mu}_0)} \quad \text{a.s.}.
        \label{eq:almost-sure-upper-bound-BBMTS}
    \end{equation}
    
    \item For each $\rho>0$, the quantity $\mathbb{E}[\tau_{\pi_{\text{BBMTS}}(\delta, \rho)}]$ satisfies
    \begin{equation}
        \limsup_{\delta \downarrow 0} \frac{\mathbb{E}[\tau_{\pi_{\text{BBMTS}}(\delta, \rho)}]}{\log(1/\delta)} \leq \frac{1+\rho}{T^*(\boldsymbol{q}_0, \boldsymbol{\mu}_0)}.
        \label{eq:expected-upper-bound-BBMTS}
    \end{equation}
\end{enumerate}
\end{theorem}
Notice that $\rho$ serves as a tuneable parameter that may be set to make the upper bound in \eqref{eq:expected-upper-bound-BBMTS} as close to \eqref{eq:lower-bound} as desired. Clearly, the right-hand side of \eqref{eq:expected-upper-bound-BBMTS} matches with the lower bound in~\eqref{eq:lower-bound} as $\rho \downarrow 0$. Hence, Theorems~\ref{prop:lower-bound} \& \ref{thm:performance-of-BBMTS} together imply that $1/T^*(\boldsymbol{q}_0, \boldsymbol{\mu}_0)$ is the optimal asymptotic growth rate of the expected stopping corresponding to the instance $(\boldsymbol{q}_0, \boldsymbol{\mu}_0)$.

\section{Partition Setting}
\label{sec:partition-setting}
In this section, we analyze the simpler setting when the~$K$ arms are {\em partitioned} among the~$M$ boxes and the learner knows $\{\mathcal{A}_m: m \in [M]\}$ but not the arm selection probabilities of the boxes. For this setting, we present an algorithm based on successive elimination, and provide a non-asymptotic, high-probability upper bound on its stopping time. We show that the upper bound is tight in the instance-specific parameters.

In the partition setting, we find it convenient to index an arm by the box it is associated with; the $k$th arm in box~$m$ is denoted by $A_{m,k}$ and its mean reward is denoted by $\mu_{m,k}.$ All arms are assumed to yield $1$-sub-Gaussian rewards. Let $\boldsymbol{q}_0=\{q_{m,k}^{0}: k \in \mathcal{A}_m, \, m \in [M]\}$ and $\boldsymbol{\mu}_0=\{\mu_{m,k}^{0}: k \in \mathcal{A}_m, \, m \in [M]\}$ define the underlying instance. Without loss of generality, let $A_{1,1}$ be the best arm in this instance. We let $\Delta_{m,k} \coloneqq  \mu_{1,1}^{0}-\mu_{m,k}^{0}$ for all $(m,k) \neq (1,1)$, and $\Delta_{1,1} \coloneqq \min_{(m,k)\neq (1,1)} \Delta_{m,k}$ denote the arm sub-optimality gaps.

\subsection{Non-Asymptotic Analysis: Achievability}
We now propose an algorithm based on successive elimination (called the {\em Boxed-Bandit Successive Elimination Algorithm} or BBSEA) and analyze its performance. Before presenting the algorithm, we introduce some notations. Our algorithm proceeds in rounds; we use $n$ to denote the round number and~$t$ to denote the running time. Let $t_{m,k}(n)$ denote the number of times arm $A_{m,k}$ is pulled up to round $n$, and let~$\hat{\mu}_{m,k}(n)$ denote the empirical mean of arm $A_{m,k}$ after~$t_{m,k}(n)$ pulls. In any given round $n$, let $S$ denote the set of {\em active arms} (candidate best arms), $S_m$ the set of {\em active arms} associated with box $m$, and $B$ the set of {\em active boxes} (boxes that are associated with at least one active arm). For a fixed $\delta \in (0,1)$, let $\alpha_\delta(x)\coloneqq \sqrt{\frac{2\log{(8Kx^2/\delta)}}{x}}$. Let ${\rm UCB}_{m,k}(n)$ and ${\rm LCB}_{m,k}(n)$ denote respectively the upper and lower confidence bounds on $\hat{\mu}_{m,k}(n)$ with confidence interval of length $\alpha_\delta(t_{m,k}(n))$, i.e.,
\begin{align}
    {\rm UCB}_{m,k}(n)
    &=\hat{\mu}_{m,k}(n) + \alpha_\delta(t_{m,k}(n)), \label{eq:ucb-expression}\\
    {\rm LCB}_{m,k}(n)
    &=\hat{\mu}_{m,k} (n) -\alpha_\delta(t_{m,k}(n)). 
    \label{eq:lcb-expression}
\end{align}
Let $\hat{a}_{\text{BBSEA}}$ denote the best arm output by BBSEA. We set the initial values $t=0$, $n=0$, $S=[K]$, $B=[M]$, and $S_m=\mathcal{A}_m$ for all $m$. In each round $n$, the following sequence of actions is executed by BBSEA: (i) Each box $m \in B$ is selected until {\em every active arm} associated with box $m$ has been pulled at least $n$ times; for every box selection, time $t$ is incremented by $1$. (ii) With every box selection and arm pull, the values of $t_{m,k}(n)$, $\hat{\mu}_{m,k}(n)$, $\mathrm{UCB}_{m,k}(n)$, and $\mathrm{LCB}_{m,k}(n)$ are updated for all the active arms. (iii) Arm $A_{m,k}$ is eliminated from $S$ in round $n$ if $\mathrm{UCB}_{m,k}(n) \leq \max_{m^\prime, k^\prime} \mathrm{LCB}_{m^\prime, k^\prime}(n)$. The above sequence of actions repeats until only one arm remains in $S$, at which point the algorithm stops and outputs the single arm in $S$ as the best arm. Because $t$ is incremented with every box selection, the stopping time is equal to the total number of box selections. See Algorithm~\ref{algo:non-asymptotic-1} for the pseudo-code of BBSEA.

\begin{algorithm}[!t]
\caption{Boxed-Bandit Successive Elimination Algorithm}
\begin{algorithmic}[1]
\renewcommand{\algorithmicrequire}{\textbf{Input:}}
 \renewcommand{\algorithmicensure}{\textbf{Output:}}
	\REQUIRE {$K$, $M$, $\delta>0$, $\mathcal{A}_m$ for $m \in [M]$}
	\ENSURE $\hat{a}_{\text{BBSEA}} \in [K]$ (best arm).\\
	\textbf{Initialization:} $B=[M]$, $n=0$, $t=0$, \\
    $S_m=\mathcal{A}_m \ \forall m$, $S=\bigcup_{m} S_m$, $\hat{\mu}_{m,k}(0)=0 \ \forall m,k$.
	\WHILE{$|S|> 1$}
        \STATE $n \leftarrow n+1$
    	\STATE For each $m \in B$, select box $m$ until every active arm $A_{m,k}$ in  box $m$ is pulled at least $n$ times. For every box selection, increment $t$ by $1$.
    	\STATE Update $t_{m,k}(n)$, $ \hat{\mu}_{m,k}(n), {\rm UCB}_{m,k}(n)$ and ${\rm LCB}_{m,k}(n)$ for all the active arms.
    	\IF{$\exists A_{m',k'} \in S$ such that ${\rm UCB}_{m,k}(n)<{\rm LCB}_{m',k'}(n)$}
    	    \STATE $S_m \leftarrow S_m\backslash A_{m,k}$,  \quad $S \leftarrow \bigcup_{m\in[M]}S_m$,
		    
		    \STATE $B \leftarrow \{m: S_m\neq \emptyset\}$.
		\ENDIF
  \IF{ $|S|=1$ }
    	    \STATE $\hat{a}_{\text{BBSEA}} \leftarrow a 
\in S$, \quad $S \leftarrow \emptyset$, \quad $B \leftarrow \emptyset$.
		\ENDIF
	\ENDWHILE
\RETURN $\hat{a}_{\text{BBSEA}}$.
\end{algorithmic}
\label{algo:non-asymptotic-1}
\end{algorithm}

	

	

\vspace{0.5em}
\noindent {\bf Performance analysis of BBSEA:} Before we present the results on the performance of BBSEA, we introduce a few useful notations. For each $(m,k)$ pair and $\delta \in (0,1)$, let $\alpha_{m,k} = 1+ \frac{102}{\Delta_{m,k}^2} \log\Big(\frac{64\sqrt{\frac{8K}{\delta}}}{\Delta_{m,k}^2}\Big)
$, and let $T_{m,k}$ denote the round number in which arm~$A_{m,k}$ is eliminated from $S$ (the set of active arms). Furthermore, let $\beta_{m,k}=\frac{1}{q_{m,k}^{0}} \bigg[\alpha_{m,k}\!+\! 2\log\frac{2K}{\delta} + 2 \sqrt{\log\frac{2K}{\delta}\left(\log\frac{2K}{\delta} + \alpha_{m,k}\right)} \bigg]$, and let $\beta_m = \max_{k \in \mathcal{A}_m} \beta_{m,k}$. 

With the above notations in place, we are now ready to state the main result of this section on the performance of BBSEA. 
\begin{theorem}
\label{thm:BBSEA_performance}
Fix $\delta \in (0,1)$, The following hold with probability greater than $1-\delta$.
\begin{enumerate}
\item BBSEA outputs the best arm correctly.
\item The stopping time of BBSEA is $ \leq \sum_{m=1}^{M} \beta_m$.
\end{enumerate}
Furthermore, for any $\pi \in \Pi(\delta)$,  
\begin{equation}
\mathbb{E}\left[\tau_\pi\right] \geq \log \left(\frac{1}{2.4\, \delta}\right) \cdot \sum_{m=1}^{M} \, \max_{k\in \mathcal{A}_m} \frac{1}{q_{m,k}^{0} \, \Delta_{m,k}^2}.
\label{eq:lower-bound-on-expected-no-of-box-selections}
\end{equation}
\end{theorem}
Notice that 
\begin{align}
    \beta_{m,k} &=O\Bigg(\frac{1}{q_{m,k}^0\, \Delta_{m,k}^2}\, \log \left(\frac{K}{\delta\Delta_{m,k}}\right)\Bigg), \label{eq:order-beta-m-k}\\
    \sum_{m=1}^{M} \beta_{m} &= \sum_{m=1}^{M} O\Bigg(\max_{k \in \mathcal{A}_m} \frac{1}{q_{m,k}^{0} \, \Delta_{m,k}^2} \log \left(\frac{K}{\delta\Delta_{m,k}}\right)\Bigg). \label{eq:order-of-upper-bound}
\end{align}
It is easy to see that the high-probability, instance-dependent upper bound on the stopping time of BBSEA given by \eqref{eq:order-of-upper-bound} matches with the instance-dependent lower bound in \eqref{eq:lower-bound-on-expected-no-of-box-selections} order-wise in~$\delta$ and the instance-specific parameters $(\boldsymbol{q}_0, \boldsymbol{\mu}_0)$. These bounds are hence tight up to multiplicative factors.

\section{Concluding Remarks}

We studied best arm identification in a multi-armed bandit when the learner's access to the arms is via exploration intermediaries that we refer to as boxes, and
the arm access probabilities of the boxes are unknown to the learner.
The key challenge we addressed in the analysis is the non-uniqueness of optimal allocations in the information theoretic lower bound. We showed that by tracking the empirical average of \emph{arbitrarily chosen} optimal allocations corresponding to running estimates of the problem instance, asymptotic optimality can be achieved. An interesting direction for future work is to develop efficient algorithms that admit non-asymptotic upper bounds; in this paper, we only do this in the special case where the arms are partitioned across the boxes. The main challenge here is to utilize estimates of arm selection probabilities of the different boxes to guide box selection. Given that any particular arm may be accessible via multiple boxes (with potentially different probabilities), it is not clear at the outset which box must be selected to maximise the chances of pulling a given arm, and therefore algorithms like LUCB do not generalize directly to our setting. 

\vspace{0.3\baselineskip}
{\noindent {\bf Acknowledgements:} Kota Srinivas Reddy was supported by the Department of Science and Technology (DST), Govt. of India, through the INSPIRE faculty fellowship. \\
P. N. Karthik was supported by the National University of Singapore via grant A-0005994-01-00. \\
Nikhil Karamchandani and Jayakrishnan Nair acknowledge support from DST via grant CRG/2021/002923 and two SERB MATRICS grants.}

\bibliographystyle{IEEEtran}
\bibliography{references}

\appendix
\appendices
Given real-valued vectors $u, v$ of identical sizes and a compact, convex set $C$, the following distance measures feature in our proofs: $d_{\infty}(u,v) \coloneqq \max_{i} |u_i - v_i|$, $d_{\infty}(u,C) \coloneqq \min_{v \in C} d_{\infty}(u,v)$, and $\| u - v \| \coloneqq \sqrt{\sum_i (u_i - v_i)^2}$.

\section{Proof of Lemma \ref{lemma:berges-plus-convexity}}
\label{appndx:proof-of-berges-plus-convexity}
We first note that for all $w \in \Sigma_M$,
\begin{align}
     &\psi(\boldsymbol{q}_0, \boldsymbol{\mu}_0, w) = \inf_{\boldsymbol{\lambda} \in \textsc{Alt}(\boldsymbol{\mu}_0)}\, \sum_{m=1}^{M} \sum_{k=1}^{K} w_m\,  q_{m,k}^{0} \, \frac{(\mu_{k}^{0}- \lambda_{k})^2}{2} \\
     &=\min_{k \neq a*(\boldsymbol{\mu}_0)}\frac{w_k^A \, w_{a^*(\boldsymbol{\mu}_0)}^A}{w_k^A+w_{a^*(\boldsymbol{\mu}_0)}^A}\frac{(\mu_{k}^{0}- \mu_{a^*({\boldsymbol{\mu}}_0)}^{0})^2}{2},
     \label{eq:psi-of-mu-and-w}
\end{align}
where $w_a^A=\sum_{m=1}^{M}w_m \, q_{m,a}^{0}$ for all  $a \in [K]$. A proof of this simply follows along the lines of \cite[Lemma 3]{kaufmann2016complexity}. From the above, it is clear that $\psi(\boldsymbol{q}_0, \boldsymbol{\mu}_0, \cdot)$ is continuous.
A simple application of the Berge's maximum theorem \cite[p. 84]{hu1997handbook} then yields that $(q,\mu) \mapsto \mathcal{W}^\star(\boldsymbol{q}, \boldsymbol{\mu})$ is upper-hemicontinuous for all $(\boldsymbol{q}, \boldsymbol{\mu})$ such that $a^\star(\boldsymbol{\mu})$ is unique.

We now show that $\mathcal{W}^\star(\boldsymbol{q}, \boldsymbol{\mu})$ is convex for all $(\boldsymbol{q}, \boldsymbol{\mu})$. Fix $(\boldsymbol{q}, \boldsymbol{\mu})$ and pick $w^{(1)}, w^{(2)} \in \mathcal{W}^\star(\boldsymbol{q}, \boldsymbol{\mu})$ and $\alpha \in [0,1]$ arbitrarily.
By definition, $$ \psi(\boldsymbol{q}, \boldsymbol{\mu}, w^{(1)}) = T^\star(\boldsymbol{q}, \boldsymbol{\mu}), \quad \psi(\boldsymbol{q}, \boldsymbol{\mu}) = T^\star(\boldsymbol{q}, \boldsymbol{\mu}). $$
We then have
\begin{align*}
    & \psi(\boldsymbol{q}, \boldsymbol{\mu}, \alpha w^{(1)} + (1-\alpha) w^{(2)}) \\
    &= \inf_{\boldsymbol{\lambda}\in \textsc{Alt}(\boldsymbol{\mu})}\!  \sum_{m=1}^{M} \sum_{k=1}^{K} \ (\alpha w_m^{(1)} + (1-\alpha) w_m^{(2)}) \,  q_{m,k}\, \frac{(\mu_{k}- \lambda_{k})^2}{2} \\
    &\geq \alpha  \inf_{\boldsymbol{\lambda} \in \textsc{Alt}(\boldsymbol{\mu})}\  \sum_{m=1}^{M} \sum_{k=1}^{K} \ w_m^{(1)} \,  q_{m,k}\, \frac{(\mu_{k}- \lambda_{k})^2}{2} \nonumber\\
    &+ (1-\alpha) \inf_{\boldsymbol{\lambda} \in \textsc{Alt}(\boldsymbol{\mu})}\  \sum_{m=1}^{M} \sum_{k=1}^{K} \ w_m^{(2)} \,  q_{m,k}\, \frac{(\mu_{k}- \lambda_{k})^2}{2}\\
    &=\alpha \, T^\star(\boldsymbol{q}, \boldsymbol{\mu}) + (1-\alpha)\, T^\star(\boldsymbol{q}, \boldsymbol{\mu})\\
    &=T^\star(\boldsymbol{q}, \boldsymbol{\mu}).
\end{align*}
On the other hand, because $\Sigma_M$ is convex, $\alpha w^{(1)} + (1-\alpha)w^{(2)} \in \Sigma_M$, and hence $\psi(\boldsymbol{q}, \boldsymbol{\mu}, \alpha w^{(1)} + (1-\alpha)w^{(2)}) \leq T^\star(\boldsymbol{q}, \boldsymbol{\mu}) $, thus implying that $\psi(\boldsymbol{q}, \boldsymbol{\mu}, \alpha w^{(1)} + (1-\alpha) w^{(2)}) = T^\star(\boldsymbol{q}, \boldsymbol{\mu})$. This proves that $\alpha w^{(1)} + (1-\alpha) w^{(2)} \in \mathcal{W}^\star(\boldsymbol{q}, \boldsymbol{\mu})$. Noting that $w^{(1)}, w^{(2)} \in \mathcal{W}^\star(\boldsymbol{q}, \boldsymbol{\mu})$ and $\alpha \in [0,1]$ are arbitrary, we arrive at the desired result.

\section{Proof of Lemma \ref{lemma:modified-tracking-with-non-unique-optimal-solutions}}
\label{appndx:proof-of-forced-exploration-lemma}
Let $C\coloneqq \mathcal{W}^\star(\boldsymbol{q}_0, \boldsymbol{\mu}_0)$. The key idea in the proof is to track the average $\bar{w}(t) \coloneqq \frac{1}{t} \sum_{s=1}^{t} w(s)$ and show that this is eventually within a neighborhood of $C$. We provide only the high-level arguments here. Fix an arbitrary $\varepsilon>0$. By virtue of the sampling rule in \eqref{eq:sampling-rule}, we have $N(t,m)=\Omega(\sqrt{t})$ a.s. for all $m$. This implies, by the strong law of large numbers, that $\|\hat{\boldsymbol{q}}(t) - \boldsymbol{q}_0 \| \to 0$ a.s., $\|\hat{\boldsymbol{\mu}}(t) - \boldsymbol{\mu}_0 \| \to 0$ a.s. as $t \to \infty$. Because $w(t+1) \in \mathcal{W}(\hat{\boldsymbol{q}}(t), \hat{\boldsymbol{\mu}}(t))$ for all $t$, and $\mathcal{W}(\hat{\boldsymbol{q}}(t), \hat{\boldsymbol{\mu}}(t))$ is ``close'' to $C$ for all $t$ large (thanks to the upper-hemicontinuity of the mapping $(\boldsymbol{q}, \boldsymbol{\mu}) \mapsto \mathcal{W}^\star(\boldsymbol{q}, \boldsymbol{\mu})$), we get that there exists $t_0=t_0(\varepsilon)$ such that $d_{\infty}(w(t), C) \leq \varepsilon$ for all $t \geq t_0$. If $v(t) \coloneqq \arg\min_{w \in C} d_{\infty}(w(t), w)$ denotes the projection of $w(t)$ onto $C$, then we have $d_{\infty}(w(t), v(t)) \leq \varepsilon$ for all $t \geq t_0$. Letting $\bar{v}(t) \coloneqq \frac{1}{t} \sum_{s=1}^{t} v(s)$, it is easy to show that $\|\bar{w}(t) - \bar{v}(t)\| \leq 9\varepsilon$ for all $t \geq t_0^\prime=t_0/(8\varepsilon)$. Letting $\hat{w}(t) \coloneqq \arg\min_{w \in C} d_{\infty}(\bar{w}(t), w)$, it then follows that $d_{\infty}(\bar{w}(t), \hat{w}(t)) \leq 9\varepsilon$ for all $t \geq t_0^\prime$.

Suppose that $\varepsilon_m(t) \coloneqq N(t,m) - t \hat{w}_m(t)$ for all $m,t$. We show below that there exists $t_0^{\prime\prime}=t_0^{\prime\prime}(\varepsilon)$ such that $\{B_{t+1}=m\} \subset \{\varepsilon_m(t) \leq 9t\varepsilon\}$ for all $t \geq t_0^{\prime\prime}$. Assuming that this is true, notice that $\varepsilon_m(t)$ satisfies the recursive relation $\varepsilon_m(t+1) = \varepsilon_m(t) + \mathbf{1}_{\{B_{t+1}=m\}} - \hat{w}_m(t+1)$. For all $t \geq t_0^{\prime\prime}$, we have $\varepsilon_m(t+1) \leq  \varepsilon_m(t) + \mathbf{1}_{\{\varepsilon_m(t) \leq 9t\varepsilon\}} - \hat{w}_m(t+1)$. We now prove by induction that $\varepsilon_m(t) \leq \max\{\varepsilon_m(t_0^{\prime\prime}), 9t\varepsilon+1\}$ for all $t \geq t_0^{\prime\prime}$ and $m$. Clearly, this holds for $t=t_0^{\prime\prime}$.
If $\varepsilon_m(t) \leq 9t\varepsilon$, then
\begin{align*}
        \varepsilon_m(t+1) 
        &\leq \varepsilon_m(t) + \mathbf{1}_{\{\varepsilon_m(t) \leq 9t\varepsilon\}} - \hat{w}_m(t+1) \nonumber\\\
        &\leq 9t\varepsilon + 1 - \hat{w}_m(t+1) \nonumber\\
        &\leq 9t\varepsilon + 1 \nonumber\\
        &\leq \max\{\varepsilon_m(t_0^{\prime\prime}), \ 9t\varepsilon + 1\}\nonumber\\
        &\leq \max\{\varepsilon_m(t_0^{\prime\prime}), \ 9(t+1)\varepsilon + 1\} \quad \forall m \in [M],
    \end{align*}
whereas if $\varepsilon_m(t) > 9t\varepsilon$, then
 \begin{align*}
        \varepsilon_m(t+1) 
        &\leq \varepsilon_m(t) - \hat{w}_m(t+1) \nonumber\\
        &\leq \max\{\varepsilon_m(t_0^{\prime\prime}), \ 9t\varepsilon + 1\} - \hat{w}_m(t+1) \nonumber\\
        &\leq \max\{\varepsilon_m(t_0^{\prime\prime}), \ 9t\varepsilon + 1\} \nonumber\\
        &\leq \max\{\varepsilon_m(t_0^{\prime\prime}), \ 9(t+1)\varepsilon + 1\} \quad \forall m \in [M].
    \end{align*}
Because $\sum_m \varepsilon_m(t) = 0$, we have 
\begin{align*}
   \varepsilon_m(t) 
   &= -\sum_{m^\prime \neq m} \varepsilon_{m^\prime}(t) \\
   &\geq -(M-1)\max_{m^\prime}\max\{\varepsilon_{m^\prime}(t_0^{\prime\prime}), \ 9(t+1)\varepsilon + 1\} 
\end{align*}
for all $m$. Thus, it follows that
\begin{align*}
    |\varepsilon_m(t)| 
    &\leq (M-1)\max_{m^\prime}\max\{\varepsilon_{m^\prime}(t_0^{\prime\prime}), \ 9t\varepsilon + 1\} \nonumber\\
    &\leq (M-1)\max\{t_0^{\prime\prime}, \ 9t\varepsilon + 1\} \nonumber\\
    &= (M-1) \, t \, \max\left\lbrace \frac{t_0^{\prime\prime}}{t}, \ 9\varepsilon + \frac{1}{t} \right\rbrace \nonumber\\
    &\leq 10(M-1)\,t\,\varepsilon
\end{align*}
for all $t \geq t_1=\frac{1}{\varepsilon}\max\{t_0^{\prime\prime}, 10\}$, which in turn yields
\begin{align*}
     d_{\infty}((N(t,m)/t)_{m \in M}, C)
    &\leq d_{\infty}((N(t,m)/t)_{m \in M}, \hat{w}(t))\nonumber\\
    &=\max_{m} \left\lvert  \frac{N(t, m)}{n} - \hat{w}_m(t)\right\rvert \nonumber\\
    & = \max_{m} \frac{|\varepsilon_m(t)|}{t} \nonumber\\
    & \leq 10(M-1) \, \varepsilon.
\end{align*}
Noting that $\varepsilon$ is arbitrary, we arrive at the desired result. 

It then remains to show that $\{B_{t+1}=m\} \subset \{\varepsilon_m(t) \leq 9t\varepsilon\}$ for all $m$ and large $t$. From the sampling rule in \eqref{eq:sampling-rule}, it is clear that $\{B_{t+1}=m\} \subset \mathcal{E}_1(t) \cup \mathcal{E}_2(t)$, where
\begin{align*}
    \mathcal{E}_1(t) \!&=\!\left\lbrace m\!=\! \arg\!\min_{\!m^\prime \in \text{supp} \left(\sum_{s=1}^{t} w(s)\right)} \! N(t, m^\prime) - \sum_{s=1}^{t} w_{m^\prime}(s)\right\rbrace,\\
    \mathcal{E}_2(t) &= \left\lbrace \min_{m^\prime \in [M]} N(t, m^\prime) < f(t) \text{ and }m=i_t \right\rbrace.
\end{align*}
Suppose that $\mathcal{E}_1(t)$ holds. Then, for all $t \geq t_0^\prime$ (with $t_0^\prime$ as defined earlier in the proof), we have
\begin{align*}
    \varepsilon_m(t) 
    &= N(t, m) - t\, \hat{w}_m(t) \nonumber\\
    &= N(t, m) - n\, \bar{w}_m(t) + t(\bar{w}_m(t) - \hat{w}_m(t)) \nonumber\\
    &\leq N(t, m) - n\, \bar{w}_m(t) + t d_{\infty}(\bar{w}(t), \hat{w}(t)) \nonumber\\
    &\leq N(t, m) - t\, \bar{w}_m(t) + 9t \varepsilon \nonumber\\
    &\stackrel{(a)}{=} \min_{m^\prime \in \text{supp} \left(\sum_{s=1}^{t} w(s)\right)} N(t, m^\prime) - t\, \bar{w}_{m^\prime}(t) + 9t\varepsilon \nonumber\\
    &\leq 9t\varepsilon,
\end{align*}
where $(a)$ above follows by noting that $\mathcal{E}_1(t)$ holds, and the last inequality follows by noting that \begin{align*}
    \sum_{m^\prime \in \text{supp} \left(\sum_{s=1}^{t} w(s)\right)} N(t, m^\prime) - t\, \bar{w}_{m^\prime}(t) \leq 0.
\end{align*}
Along similar lines as above, it can be shown that $\varepsilon_m(t) \leq 9t\varepsilon$ whenever $\mathcal{E}_2(t)$ holds. This completes the proof. 

\section{Proof of Lemma \ref{lemma:explicit-expression-for-Z-a-b}}
Given $(\boldsymbol{q}, \boldsymbol{\mu})$ and $t \geq 1$, the log-likelihood of all the box selections, arm pulls, and observations from the arms up to $t$ under the instance $C=(\boldsymbol{q}, \boldsymbol{\mu})$ is given by
\begin{align}
    Z(t)
    &=\log P(B_1) + \sum_{s=2}^{t} \log P(B_s|\mathcal{F}_{s-1})\nonumber\\
    &+\sum_{m=1}^{M} \sum_{k=1}^{K} N(t, m, k) \log q_{m,k} \nonumber\\
    &+ \sum_{m=1}^{M} \sum_{k=1}^{K} \sum_{s=1}^{t} \mathbf{1}_{\{A_s=k\}} \log \left(\beta \exp\left(-\frac{(X_s-\mu_{k})^2}{2}\right)\right),
    \label{eq:LL-1}
\end{align}
where $\beta$ is a normalisation constant of the Gaussian distribution with variance $1$.
Noting that the first two terms in \eqref{eq:LL-1} do not depend on $C$ (as any policy $\pi$ is agnostic to the knowledge of the underlying problem instance), it follows from \eqref{eq:LL-1} that 
\begin{align}
    &\sup_{C=(\boldsymbol{q}, \boldsymbol{\mu}): \mu_a \geq \mu_b} \exp(Z_C(t)) \nonumber\\
    &\propto \sup_{C=(\boldsymbol{q}, \boldsymbol{\mu}): \mu_a \geq \mu_b} \exp \bigg(\sum_{m=1}^{M}\ \sum_{k=1}^{K}\ N(t, m, k) \log q_{m,k}\nonumber\\
    &\hspace{3cm}-\sum_{k=1}^{K}\ \sum_{s=1}^{t} \mathbf{1}_{\{A_s=k\}} \frac{(X_s-\mu_{k})^2}{2}\bigg) \nonumber\\
    &= \exp \bigg( \sum_{m=1}^{M}\ \sum_{k=1}^{K}\ N(t, m, k) \log \hat{q}_{m,k}(t) \nonumber\\
    &\hspace{1cm}- \inf_{\mu: \mu_a \geq \mu_b} \sum_{k=1}^{K}\ \sum_{s=1}^{t} \mathbf{1}_{\{A_s=k\}} \frac{(X_s-\mu_{k})^2}{2} \bigg),
    \label{eq:proof-of-explicit-expression-1}
\end{align}
where $\{\hat{q}_{m,k}(t)\}_{m,k}$ are the empirical estimates of the box probabilities at time $t$.
We now note that the minimisation
\begin{equation}
     \inf_{\boldsymbol{\mu}: \mu_a \geq \mu_b} \sum_{k=1}^{K}\ \sum_{s=1}^{t} \mathbf{1}_{\{A_s=k\}} \frac{(X_s-\mu_{k})^2}{2}
     \label{eq:convex-program}
\end{equation}
is a simple convex program whose optimal solution can be shown to be
\begin{equation}
    \mu_k^\prime = 
    \begin{cases}
        \hat{\mu}_k(t), & \text{if } \hat{\mu}_a(t) >   \hat{\mu}_b(t), \\
        \hat{\mu}_{a,b}(t), & \text{otherwise}
    \end{cases} \quad \text{for } k=a,b,
    \label{eq:overall-optimal-solution-1}
\end{equation}
and $\mu_k^\prime = \hat{\mu}_k(t)$ for all $k \neq a,b$. Similarly, it can be shown that the optimal solution to $$\inf_{\boldsymbol{\mu}: \mu_a \leq \mu_b} \sum_{k=1}^{K}\ \sum_{s=1}^{t} \mathbf{1}_{\{A_s=k\}} \frac{(X_s-\mu_{k})^2}{2},$$
say $\mu^{\prime\prime} = (\mu_k^{\prime\prime}: k \in [K])$, is
\begin{equation}
    \mu_k^{\prime\prime} = 
    \begin{cases}
        \hat{\mu}_k(t), & \text{if } \hat{\mu}_a(t) <   \hat{\mu}_b(t), \\
        \hat{\mu}_{a,b}(t), & \text{otherwise}
    \end{cases} \quad \text{for } k=a,b,
    \label{eq:overall-optimal-solution-2}
\end{equation}
and $\mu_k^{\prime\prime} = \hat{\mu}_k(t)$ for all $k \neq a,b$.
We now prove Lemma~\ref{lemma:explicit-expression-for-Z-a-b}. Suppose that $\hat{\mu}_a(t) \geq \hat{\mu}_b(t)$. Then, noting that the first term inside the exponential in \eqref{eq:proof-of-explicit-expression-1} is common to both the numerator and the denominator terms of the generalised log-likelihood ratio \eqref{eq:GLRT-statistic} (and hence cancels out), we get
\begin{align}
    Z_{a, b}(t) 
    &= -\sum_{k=1}^{K}\ \sum_{s=1}^{t} \mathbf{1}_{\{A_s=k\}} \frac{(X_s-\mu_k^{\prime})^2}{2} \nonumber\\
    &\hspace{0.7cm}+ \sum_{k=1}^{K}\ \sum_{s=1}^{t} \mathbf{1}_{\{A_s=k\}} \frac{(X_s-\mu_k^{\prime\prime})^2}{2} \nonumber\\
    &= \sum_{s=1}^{t} \mathbf{1}_{\{A_s=a\}} \frac{(X_s-\hat{\mu}_{a, b}(t))^2 - (X_s-\hat{\mu}_{a}(t))^2}{2} \nonumber\\
    &\hspace{0.7cm}+ \sum_{s=1}^{t} \mathbf{1}_{\{A_s=b\}} \frac{(X_s-\hat{\mu}_{a, b}(t))^2 - (X_s-\hat{\mu}_{b}(t))^2}{2} \nonumber\\
    & = N_a(t) \, \frac{\left( \hat{\mu}_a(t) - \hat{\mu}_{a, b}(t) \right)^2}{2} \nonumber\\
    &\hspace{0.7cm}+ N_b(t) \, \frac{\left( \hat{\mu}_b(t) - \hat{\mu}_{a, b}(t) \right)^2}{2}.
    \label{eq:proof-of-explicit-expression-3}
\end{align}
When $\hat{\mu}_a(t) \leq \hat{\mu}_b(t)$, it can be shown along similar lines as above that $Z_{a, b}(t) = -Z_{b, a}(t)$, where $Z_{b, a}(t)$ is as in \eqref{eq:proof-of-explicit-expression-3} with the roles of $a$ and $b$ interchanged.

\section{Proof of Theorem \ref{thm:performance-of-BBMTS}}
Throughout the proof, we use the symbol $\pi$ to denote $\pi_{\text{BBMTS}(\delta, \rho)}$ for fixed parameters $\delta,\rho$. Let $\hat{a}$ denote the best arm ouput by BBMTS at stoppage.

{\em Proof of part 1:} Let
\begin{align*}
    \mathcal{E}_1 &= \left\lbrace \exists t \geq 1: Z(t) \geq \zeta(t, \delta, \rho) \ \text{and} \min_{k \in [K]} N_k(t) > 0 \right\rbrace,\\
    \mathcal{E}_2 &= \{\mu_{\hat{a}} < \mu_{a^\star(\boldsymbol{\mu}_0)}\}.
\end{align*}
Recall that $\hat{a}(t) \in \arg\max_{a \in [K]} \hat{\mu}_a(t)$ denotes the estimate of the best arm at time $t$, and $\hat{a}(\tau_\pi) = \hat{a}$. We then have
\begin{align}
    &P(\mathcal{E}_1 \cap \mathcal{E}_2) \nonumber\\
    &\leq  \sum_{t=1}^{\infty}P\left( Z(t) \geq \zeta(t, \delta, \rho), \, \min_{k} N_k(t) \!>\! 0, \, \mu_{\hat{a}(t)} \!<\! \mu_{a^\star(\boldsymbol{\mu}_0)} \right) \nonumber\\
    &\leq  \sum_{t=1}^{\infty} P\left(Z_{\hat{a}(t), a^\star(\boldsymbol{\mu}_0)} \geq \zeta(t, \delta, \rho),  \ \min_k N_k(t)>0\right) \nonumber\\
    &= \sum_{t=1}^{\infty} \ P\left(\sum_{k=1}^{K} N_k(t)\, \frac{(\hat{\mu}_k(t) - \mu_k)^2}{2} \geq \zeta(t, \delta, \rho)\right) \nonumber\\
    &\stackrel{(a)}{\leq} \sum_{t=1}^{\infty} e^{K+1} \left(\frac{\zeta(t, \delta, \rho)^2\, \log t}{K}\right)^K \, e^{-\zeta(t, \delta, \rho)} \nonumber\\
    &\leq \delta,
    \label{eq:proof-of-delta-PAC-property-1}
\end{align}
where $(a)$ above is due to \cite[Theorem 2]{magureanu2014lipschitz} and the last line follows from the definition of $\zeta(t, \delta, \rho)$. This completes the proof of part 1.

{\em Proof of part 2:} 
We use the below technical lemma from \cite{garivier2016optimal} in the proof.
\begin{lemma}
\label{lemma:technical-lemma}\cite[Lemma 18]{garivier2016optimal}
For any two constants $c_1>0, c_2>0$ such that $c_2/c_1 \geq 1$ and $\alpha \geq 1$, 
\begin{equation*}
    \inf\{t \geq 1: t c_1 \geq \log(c_2 t^\alpha)\} \leq \frac{\alpha}{c_1} \left(\log\frac{c_2 e}{c_1^\alpha} + \log\log \frac{c_2}{c_1^\alpha}\right).
\end{equation*}
\end{lemma}
Consider the event
\begin{align*}
    \mathcal{E} &= \bigg\lbrace d_{\infty}((N(t, m)/t)_{m \in [M]}, \ \mathcal{W}^\star(\boldsymbol{q}_0, \boldsymbol{\mu}_0)) \stackrel{t \to \infty}{\longrightarrow} 0,\\
    &\hspace{3cm} \hat{\boldsymbol{\mu}}(t) \stackrel{t \to \infty}{\longrightarrow} \boldsymbol{\mu}, \
    \hat{\boldsymbol{q}}(t) \stackrel{t \to \infty}{\longrightarrow} \boldsymbol{q}_0 \bigg\rbrace.
\end{align*}
Thanks to Lemma~\ref{lemma:modified-tracking-with-non-unique-optimal-solutions}, we know that $\mathcal{E}$ occurs with probability $1$ under BBMTS.
Fix $\varepsilon>0$ and $w^\star \in \mathcal{W}^\star(q, \mu)$ arbitrarily. From the continuity of $\psi$ (as defined in \eqref{eq:psi-of-mu-and-w}) established in Appendix~\ref{appndx:proof-of-berges-plus-convexity}, we know that there exists an open neighbourhood of $\{\boldsymbol{q}_0\} \times \{\boldsymbol{\mu}_0\} \times \mathcal{W}^\star(\boldsymbol{q}_0, \boldsymbol{\mu}_0)$, say $\mathcal{U}=\mathcal{U}(\varepsilon)$, such that
\begin{equation}
    \psi(\boldsymbol{q}, \boldsymbol{\mu}, w^\prime) \geq (1-\varepsilon)\, \psi(\boldsymbol{q}_0, \boldsymbol{\mu}_0, w^\star) = (1-\varepsilon)\, T^\star(\boldsymbol{q}_0, \boldsymbol{\mu}_0).
\end{equation}
for all $(\boldsymbol{q}, \boldsymbol{\mu}, w^\prime) \in \mathcal{U}$. Under the event $\mathcal{E}$, there exists $n_0 = n_0(\varepsilon)$ such that $(\hat{\boldsymbol{\mu}}(t), \hat{\boldsymbol{q}}(t), (N(t, m)/t)_{m \in [M]}) \in \mathcal{U}$ for all $t \geq n_0$, and consequently
\begin{equation}
    \psi(\hat{\boldsymbol{q}}(t), \hat{\boldsymbol{\mu}}(t), (N(t, m)/t)_{m \in [M]}) \geq (1-\varepsilon)\, T^\star(\boldsymbol{q}_0, \boldsymbol{\mu}_0)
\end{equation}
for all $t \geq n_0$. Because $\|\hat{\boldsymbol{\mu}}(t)-\boldsymbol{\mu}_0\| \stackrel{t \to \infty}{\longrightarrow} 0$ under $\mathcal{E}$, and  $a^\star(\boldsymbol{\mu}_0)$ is unique, there exists $n_1$ large such that $\hat{a}(t) = \arg\max_{a \in [K]} \hat{\mu}_a(t)$ is unique $\forall t \geq n_1$, in which case we may express the test statistic $Z(t)$ as $Z(t) 
    = \min_{b \neq \hat{a}(t)} Z_{\hat{a}(t), b}(t)
    = t \, \psi(\hat{\boldsymbol{q}}(t), \hat{\boldsymbol{\mu}}(t), (N(t, m)/t)_{m \in [M]})$.
Thanks to the sampling rule \eqref{eq:sampling-rule}, we have $\min_{m \in [M]} N(t, m) =\Omega(\sqrt{t})>0$ a.s. for all $t$ large. This, along with the fact that~$\|\hat{\boldsymbol{q}}(t) - \boldsymbol{q}_0\| \stackrel{t\to\infty}{\longrightarrow} 0$ under the event $\mathcal{E}$, implies that there exists $n_2 \geq 1$ such that $\min_{k \in [K]} N_k(t) > 0$ for all $t \geq n_2$. Thus, under $\mathcal{E}$, for all $t \geq \max\{n_0, n_1, n_2\}$, we have
\begin{equation}
    Z(t) \geq t (1-\varepsilon) \, T^\star(\boldsymbol{q}_0, \boldsymbol{\mu}_0), \quad \min_{k \in [K]} N_k(t) > 0.
    \label{eq:proof-of-almost-sure-upper-bound-1}
\end{equation}
We then have (format from here onwards)
\begin{align}
    &\tau_{\pi} = \inf\left\lbrace t \geq 1: Z(t) \geq \zeta(t, \delta, \rho),\ \min_{k \in [K]} N_k(t) > 0 \right\rbrace \nonumber\\
    & \leq \inf\left\lbrace t \geq \max\{n_0, n_1, n_2\}: Z(t) \geq \zeta(t, \delta, \rho) \right\rbrace \nonumber\\ 
    &\leq \inf\bigg\lbrace t \geq \max\{n_0, n_1, n_2\}: t (1-\varepsilon)\, T^\star(\boldsymbol{q}_0, \boldsymbol{\mu}_0) \\
    &\hspace{5cm} \geq \log \frac{C\, t^{1+\rho}}{\delta}\bigg\rbrace \nonumber\\
    &\leq \max\left\lbrace \max\{n_0, n_1, n_2\}, \  \inf\left\lbrace t \geq 1: t c_1 \geq  \log (c_2 t^{1+\rho}) \right\rbrace \right\rbrace \nonumber\\
    &\leq \max\left\lbrace n_0, n_1, n_2,\frac{1}{c_1} \left(\log\frac{c_2 e}{c_1^{1+\rho}} + \log\log \frac{c_2}{c_1^{1+\rho}}\right) \right\rbrace,
    \label{eq:proof-of-almost-sure-upper-bound-2}
\end{align}
for all $\delta$ sufficiently small, where the last line above follows from the application of Lemma \ref{lemma:technical-lemma} with
\begin{equation}
    c_1 = \frac{T^\star(q, \mu) (1-\varepsilon)}{1+\rho}, \quad c_2 = \frac{1}{\delta}, \quad \alpha=1+\rho.
    \label{eq:c-1-and-c-2}
\end{equation}
Eq. \eqref{eq:proof-of-almost-sure-upper-bound-2} immediately yields $\tau_{\pi} < \infty$ almost surely. Dividing by $\log(1/\delta)$ and taking limits as $\delta \downarrow 0$ in \eqref{eq:proof-of-almost-sure-upper-bound-2}, we get
\begin{equation}
    \limsup_{\delta \downarrow 0} \frac{\tau_{\pi}}{\log(1/\delta)} \leq \frac{1+\rho}{T^\star(q, \mu) (1-\varepsilon)} \quad a.s..
    \label{eq:proof-of-almost-sure-upper-bound-3}
\end{equation}
Noting that \eqref{eq:proof-of-almost-sure-upper-bound-3} holds for all $\varepsilon>0$, and taking limits as $\varepsilon \downarrow 0$, we arrive at the desired result.

{\em Proof of part 3:} 
Fix $\varepsilon>0$ and $w^\star \in \mathcal{W}^\star(\boldsymbol{q}_0, \boldsymbol{\mu}_0)$. By virtue of the continuity of $\psi$, there exists $\xi_1(\varepsilon)>0$ such that for all $(\boldsymbol{q}, \boldsymbol{\mu}, w^\prime)$ satisfying
\begin{equation}
    \|\boldsymbol{q}  - \boldsymbol{q}_0\| \leq \xi_1(\varepsilon), \ \|\boldsymbol{\mu} - \boldsymbol{\mu}_0\| \leq \xi_1(\varepsilon), \ d_{\infty}(w^\prime, w^\star) \leq \xi_1(\varepsilon),
    \label{eq:proof-of-upper-bound-expected-value-1}
\end{equation}
we have $|\psi(\boldsymbol{q}, \boldsymbol{\mu}, w^\prime) - \psi(\boldsymbol{q}_0, \boldsymbol{\mu}_0, w^\star)| \leq \varepsilon \, \psi(\boldsymbol{q}_0, \boldsymbol{\mu}_0, w^\star) = \varepsilon \, T^\star(\boldsymbol{q}_0, \boldsymbol{\mu}_0)$. Additionally, from the upper-hemicontinuity of $(\boldsymbol{q}, \boldsymbol{\mu}) \mapsto \mathcal{W}^\star(\boldsymbol{q}, \boldsymbol{\mu})$ (Lemma~\ref{lemma:berges-plus-convexity}), there exists $\xi_2(\varepsilon)>0$ s.t. 
\begin{align*}
    &\|\boldsymbol{q}  - \boldsymbol{q}_0\| \leq \xi_2(\varepsilon),\, \|\boldsymbol{\mu}  - \boldsymbol{\mu}_0\| \leq \xi_2(\varepsilon) \nonumber\\
    &\implies d_{\infty}(w^\prime, \mathcal{W}^\star(\boldsymbol{q}_0, \boldsymbol{\mu}_0)) \leq \frac{\xi_1(\varepsilon)}{10(M-1)} \ \forall w^\prime \in \mathcal{W}^\star(\boldsymbol{q}, \boldsymbol{\mu}).
\end{align*}
Set $\xi(\varepsilon) = \min\{\xi_1(\varepsilon), \xi_2(\varepsilon)\}$. For $N > 1$, let
\begin{equation}
    \mathcal{G}_{N}\!=\!  \bigcap_{t=N-1}^{\infty} \{\|\boldsymbol{q} - \hat{\boldsymbol{q}}(t)\| \leq \xi(\varepsilon), \|\boldsymbol{\mu} - \hat{\boldsymbol{\mu}}(t)\| \leq \xi(\varepsilon)\}.
    \label{eq:event-G-N}
\end{equation}
Under the event $\mathcal{G}_N$, we note that for all $t \geq N$,
\begin{align}
    &d_{\infty}(w(t), \mathcal{W}^\star(\boldsymbol{q}, \boldsymbol{\mu})) \\
    &\leq \max_{w^\prime \in \mathcal{W}^\star(\hat{q}(t-1), \hat{\mu}(t-1))} d_{\infty}(w^\prime, \mathcal{W}^\star(\boldsymbol{q}, \boldsymbol{\mu})) \nonumber\\
    &\leq \frac{\xi_1(\varepsilon)}{10(M-1)}.
\end{align}
Furthermore, defining $\varepsilon_1 \coloneqq \xi_1(\varepsilon)/(10(M-1))$, it follows from Lemma~\ref{lemma:modified-tracking-with-non-unique-optimal-solutions} that there exists $n_1=n_1(\varepsilon_1)$ such that for all $ t \geq n_1$, a.s.,
\begin{equation}
    d_{\infty}((N(t, m)/t)_{m \in [M]}, \ \mathcal{W}^\star(\boldsymbol{q}, \boldsymbol{\mu})) \leq 10(M-1)\varepsilon_1 \leq \xi_1(\varepsilon).
    \label{eq:proof-of-upper-bound-expected-value-2}
\end{equation}
Setting $N_0=n_1$, and using \eqref{eq:proof-of-upper-bound-expected-value-1} and \eqref{eq:proof-of-upper-bound-expected-value-2}, we see that $\mathcal{G}_N$ is a subset of the event
$$
\bigcap_{t=N}^{\infty} \bigg\lbrace \psi(\hat{\boldsymbol{q}}(t), \hat{\boldsymbol{\mu}}(t), (N(t, m)/t)_{m \in [M]}) \geq (1-\varepsilon) T^\star(\boldsymbol{q}, \boldsymbol{\mu}) \bigg\rbrace
$$ 
for all $N \geq N_0$. Let $N_1 = \inf\{t \geq 1: \min_{k \in [K]} N_k(t) > 0\}$. Recall that $Z(t) = t \psi(\hat{\boldsymbol{q}}(t), \hat{\boldsymbol{\mu}}(t), (N(t, m)/t)_{m \in [M]})$ for all $t \geq N_1$. Then, for all $N\geq \max\{N_0,N_1\}$, we have under $\mathcal{G}_N$
\begin{align}
    &\tau_{\pi} = \inf\left\lbrace t \geq 1: Z(t) \geq \zeta(t, \delta, \rho), \ \min_{k \in [K]} N_k(t)>0 \right\rbrace \nonumber\\
    &\leq \inf\left\lbrace t \geq N: Z(t) \geq \zeta(t, \delta, \rho) \right\rbrace \nonumber\\
    & \leq \inf\left\lbrace t \geq N: t\,c_1 \geq \log (c_2\, t) \right\rbrace \nonumber\\
    & \leq \max\left\lbrace N, \ \inf\left\lbrace t \geq 1: c_1 t \geq \log (c_2 t^{1+\rho})\right\rbrace \right\rbrace \nonumber\\
    &\leq \max\left\lbrace N, \  \frac{1+\rho}{c_1} \left(\log\frac{c_2 e}{c_1^{1+\rho}} + \log\log \frac{c_2}{c_1^{1+\rho}}\right) \right\rbrace,
    \label{eq:proof-of-upper-bound-expected-value-4}
\end{align}
where the last line above follows from Lemma~\ref{lemma:technical-lemma}, with $c_1$, $c_2$, and $\alpha$ as in \eqref{eq:c-1-and-c-2}. Also, in \eqref{eq:proof-of-upper-bound-expected-value-4}, we have $\log\log \frac{c_2}{c_1}=o(\log(1/\delta))$. Letting
$$
N_2(\delta, \rho) \coloneqq \frac{1+\rho}{(1-\varepsilon)\, T^\star(q, \mu)} \log\frac{1}{\delta} + o(\log(1/\delta)),
$$
it follows from \eqref{eq:proof-of-upper-bound-expected-value-4} that 
\begin{equation}
    \mathcal{G}_N \subset \{\tau_{\pi} \leq N\} \quad \forall N\geq \max\{N_0, \,  N_1, \, N_2(\delta, \rho)\}.
    \label{eq:subset-relation-of-interest}
\end{equation}
Setting $N_3(\delta, \rho) = \max\{N_0,\, N_1, \, N_2(\delta, \rho)\}$, we then have
\begin{align}
    \tau_{\pi} 
    &= \tau_{\pi}\,\mathbf{1}_{\{\tau_{\pi} \leq N_3(\delta, \rho)\}} + \tau_{\pi}\,\mathbf{1}_{\{\tau_{\pi} > N_3(\delta, \rho)\}} \nonumber\\
    & \leq N_3(\delta, \rho) + \max\{N_3(\delta, \rho), \,  \tau_{\pi}\} \nonumber\\
    &\leq N_2(\delta, \rho) + \max\{N_0, \, N_1\} + \tau,
    \label{eq:proof-of-upper-bound-expected-value-5}
\end{align}
where in writing \eqref{eq:proof-of-upper-bound-expected-value-5}, we use the relation $\max\{a, b, c\} \leq c + \max\{a,b\}$, and $\tau \coloneqq \max\{N_3(\delta, \rho), \,  \tau_{\pi}\}$. Taking expectations on both sides of \eqref{eq:proof-of-upper-bound-expected-value-5}, we get
\begin{equation}
    \mathbb{E}[\tau_{\pi}] \leq N_2(\delta, \rho) + \max\{N_0, \, N_1\} +  \mathbb{E}[\tau].
    \label{eq:proof-of-upper-bound-expected-value-6}
\end{equation}
We now note that
\begin{align}
    \mathbb{E}[\tau] 
    &= \sum_{N=0}^{\infty} P(\max\{N_3(\delta, \rho), \, \tau_{\pi}\} > N) \nonumber\\
    &=\sum_{N=N_3(\delta, \rho)}^{\infty} P(\tau_{\pi} > N) \nonumber\\
    &\leq \sum_{N=N_3(\delta, \rho)}^{\infty} P(\mathcal{G}_N^{\mathsf{c}}) \nonumber\\
    &\leq \sum_{N=\max\{N_0, \, N_1\}}^{\infty} P(\mathcal{G}_N^{\mathsf{c}}),
    \label{eq:proof-of-upper-bound-expected-value-7}
\end{align}
where the penultimate line is due to \eqref{eq:subset-relation-of-interest}. 

We now demonstrate below that the summation in \eqref{eq:proof-of-upper-bound-expected-value-7} is finite and may be upper bounded by a constant that is independent of $\delta$. Indeed, from \eqref{eq:event-G-N}, we have
\begin{align}
    P(\mathcal{G}_N^{\mathsf{c}}) 
    &\leq \sum_{t=N-1}^{\infty} P(\|\hat{\boldsymbol{q}}(t) - \boldsymbol{q}_0\| > \xi(\varepsilon)) \nonumber\\
    &\hspace{1cm}+ \sum_{t=N-1}^{\infty} P(\|\hat{\boldsymbol{\mu}}(t) - \boldsymbol{\mu}_0\| > \xi(\varepsilon)) \nonumber\\
    &\stackrel{(a)}{\leq} \sum_{t=N-1}^{\infty} c_1 \, t^{M/4} \exp(-c_2 \xi(\varepsilon)^2 \sqrt{t}) \nonumber\\
    &\hspace{1cm}+ \sum_{t=N-1}^{\infty} c_3 \, t^{K/4} \exp(-c_4 \xi(\varepsilon)^2 \sqrt{t}) \nonumber\\
    &\leq \int_{N-2}^{\infty} c_1 \, s^{M/4} \exp(-c_2 \xi(\varepsilon)^2 \sqrt{s})\, ds \nonumber\\
    &\hspace{1cm}+ \int_{N-2}^{\infty} c_3 \, s^{K/4} \exp(-c_4 \xi(\varepsilon)^2 \sqrt{s})\, ds \nonumber\\
    &= O\left(\frac{(N-2)^{M/2 + 1}}{\xi(\varepsilon)^4 \, \exp(c_2 \, \xi(\varepsilon)^2 \, \sqrt{N-2})}\right) \nonumber\\
    &\hspace{1cm}+ O\left(\frac{(N-2)^{K/2 + 1}}{\xi(\varepsilon)^4 \, \exp(c_4 \, \xi(\varepsilon)^2 \, \sqrt{N-2})}\right),
    \label{eq:proof-of-upper-bound-expected-value-8}
\end{align}
where $(a)$ follows from standard concentration inequalities for sub-Gaussian random variables \cite{boucheron2013concentration}, and the last line follows by upper bounding the integrals in the penultimate line by incomplete Gamma functions and upper bounding these functions using some classical inequalities \cite{natalini2000inequalities, borwein2009uniform}. Noting that the summation of the right-hand side of \eqref{eq:proof-of-upper-bound-expected-value-8} over $N \geq \max\{N_0, \, N_1\}$ is finite (equal to $b$, say) and independent of $\delta$, using \eqref{eq:proof-of-upper-bound-expected-value-7} and \eqref{eq:proof-of-upper-bound-expected-value-8} in \eqref{eq:proof-of-upper-bound-expected-value-6}, we have 
\begin{align}
    &\frac{\mathbb{E}[\tau_{\pi}]}{\log(1/\delta)} 
    \leq \frac{N_2(\delta, \rho)}{\log(1/\delta)} + \frac{\max\{N_0 \, \varepsilon_1^3, \, N_1\} + b}{\log(1/\delta)} \nonumber\\
    &= \frac{1+\rho}{(1-\varepsilon)\, T^\star(q, \mu)} + \frac{o(\log(1/\delta))+\max\{N_0 \, \varepsilon_1^3, \, N_1\} + b}{\log(1/\delta)}.
\end{align}
Taking limits as $\delta \downarrow 0$ followed by $\varepsilon \downarrow 0$, we arrive at \eqref{eq:expected-upper-bound-BBMTS}.

\section{Proof of Theorem~\ref{thm:BBSEA_performance}}
To prove Theorem \ref{thm:BBSEA_performance}, we use the following concentration inequalities. For the proofs, we refer the reader to \cite{boucheron2013concentration}.
\begin{lemma}
\label{lemma:conc-ineq}
Let $X_1,X_2,\ldots,X_n$ be $n$ i.i.d. sub-Gaussian random variables with mean $\mu$ and variance $\sigma^2$. Then, for any $\epsilon>0$, we have
\begin{equation}
P\left(\left|\frac{\sum_{i=1}^{n}X_i}{n}-\mu \right|> \epsilon\right)\leq 2 \exp \left \lbrace -\frac{ n\epsilon^2}{2\sigma^2} \right \rbrace. \label{eqn:conc}
\end{equation}
\end{lemma}

\begin{lemma}
\label{lemma:conc-ineq2}
Let $X_1,X_2,\ldots,X_n$ be $n$ i.i.d. Bernoulli random variables with mean $p$. Then, for any $0<\gamma<1$, we have
\begin{equation}
P\left( {\sum_{i=1}^{n}X_i} \leq  (1-\gamma) np \right)\leq \exp\left(-\frac{np \gamma^{2}}{2}\right). \label{eqn:conc2}
\end{equation}
\end{lemma}
The following result formalises that for all $(m,k)$ pairs, the empirical mean $\hat{\mu}_{m,k}(n)$ of the arm $A_{m,k}$ lies between its upper and lower confidence bounds defined in Section \ref{sec:partition-setting} with high probability.
\begin{lemma}\label{lemma:event_1}
    Let $\mathcal{E}_1$ be the following event:
    \begin{align*}
	\!\!\mathcal{E}_1\! \coloneqq\! \bigcap_{\substack{n\in\mathbb{N}}}\, \bigcap_{m=1}^{M} \bigcap_{k \in \mathcal{A}_m}\{{\rm LCB}_{m,k}(n) \leq  {\mu}_{m,k}(n) \leq  {\rm UCB}_{m,k}(n)\}.
	\label{eq:event-E1}
\end{align*}
Then, $P(\mathcal{E}_1)\geq 1-\frac{\delta}{2}$.
\end{lemma}
The proof of Lemma~\ref{lemma:event_1} uses Lemma~\ref{lemma:conc-ineq} and is quite standard; see, for instance, \cite[Lemma 2]{reddy2022almost} for a proof template.




\begin{lemma}\label{lemma:no_of_arm_polls}
    Conditioned on $\mathcal{E}_1$, the following hold w.p.1.
    \begin{itemize}
        \item BBSEA outputs the correct best arm.
        \item $T_{m,k} \leq \alpha_{m,k}$ for all $m \in [M]$ and $k\in \mathcal{A}_m$.
    \end{itemize}
\end{lemma}
Lemma~\ref{lemma:no_of_arm_polls} informs that (a) as long as the arm means lie within their respective confidence intervals, BBSEA gives correct output, and (b) arm $A_{m,k}$ is identified as optimal or sub-optimal in at most $\alpha_{m,k}$ rounds of BBSEA. Again, the proof of Lemma~\ref{lemma:no_of_arm_polls} is quite standard and omitted; see, for instance, \cite[Theorem 1]{reddy2022almost} and \cite[Theorem 2]{reddy2022almost} for a proof template.

\begin{lemma}\label{lemma:event_2}
    Let $\mathcal{E}_2$ be the event that $\beta_{m,k}$ selections of box $m$ gives rise to at least $\alpha_{m,k}$ pulls of arm $A_{m,k}$ for all $m \in [M]$ and $k \in \mathcal{A}_m$. Then, $P(\mathcal{E}_2)\geq 1-\frac{\delta}{2}$.
\end{lemma}
\begin{IEEEproof}[Proof of Lemma~\ref{lemma:event_2}]
Recall from our problem setup that when the underlying instance is $(\boldsymbol{q}_0, \boldsymbol{\mu}_0)$, selection of box $m$ leads to the pulling of arm $A_{m,k}$ with probability $q_{m,k}^{0}$. Fix $(m,k)$, and define
\begin{align*}
    X_i=\begin{cases}
        1, \!&\! A_{m,k} \text{ is pulled at the }i \text{th selection of box }m, \\
        0, \!&\! A_{m,k} \text{ is not pulled at the }i \text{th selection of box }m.
    \end{cases}
\end{align*}
Clearly, $X_1, X_2, \ldots$ are i.i.d. and follow Bernoulli distribution with mean~$q_{m,k}^{0}$. Applying Lemma~\ref{lemma:conc-ineq2} with $n=\beta_{m,k}$, 
\begin{align*}
    & \gamma = \frac{2 \sqrt{\ln\left(\frac{2K}{\delta}\right)}}{\sqrt{\ln\left(\frac{2K}{\delta}\right)+\alpha_{m,k}}+\sqrt{\ln\left(\frac{2K}{\delta}\right)}},
\end{align*}
and $p=q_{m,k}^0$, we get that
$\beta_{m,k}$ selections of box $m$  gives rise to $\alpha_{m,k}$ pulls of arm $A_{m,k}$ with probability greater than $1-\frac{\delta
}{2K}.$ The proof is completed by applying union bound over all $(m,k)$ pairs, and noting that $\sum_{m=1}^{M} |\mathcal{A}_m| = K$.
\end{IEEEproof}

With the above ingredients in place, we proceed to prove Theorem~\ref{thm:BBSEA_performance}. Note that 
$
P(\mathcal{E}_1 \cap \mathcal{E}_2) \geq 1-\delta.
$
Part~1 of Theorem~\ref{thm:BBSEA_performance} follows directly from Lemma~\ref{lemma:no_of_arm_polls}.
To prove the second part, let $U_{m,k}^\prime$ denote the random number of selections of box $m$ that are required to pull arm $A_{m,k}$ at least $T_{m,k}$ times. Let $U_m^\prime \coloneqq \max_{k \in \mathcal{A}_m} U_{m,k}^\prime$. 
\begin{lemma}
\label{lemma:final}
    On the event $\mathcal{E}_1 \cap \mathcal{E}_2$, the following hold for all $m \in [M]$ and $k \in \mathcal{A}_m$ w.p.1.
    \begin{enumerate}
        \item $U'_{m,k}\leq \beta_{m,k}$.
        
        \item $U'_m\leq \beta_m$.
    \end{enumerate}
\end{lemma}
\begin{IEEEproof}[Proof of Lemma~\ref{lemma:final}]
    Recall that under the event $\mathcal{E}_2$, $\beta_{m,k}$ selections of box $m$  gives rise to $\alpha_{m,k}$ pulls of arm $A_{m,k}$ for all $(m,k)$. Using this and the fact that $T_{m,k}\leq \alpha_{m,k}$ for all $(m,k)$ under $\mathcal{E}_1$ (Lemma~\ref{lemma:no_of_arm_polls}), we deduce that $U_{m,k}^\prime \leq \beta_{m,k}$, thus proving part~1. Part~2 simply follows by applying maximum over $k \in \mathcal{A}_m$ to the result in part~2 and using the definitions of $U_m^\prime$ and $\beta_m$ for each $m \in [M]$.
\end{IEEEproof}
Parts 2 of Theorem~\ref{thm:BBSEA_performance} follows immediately from Lemma~\ref{lemma:final} and the fact that the stopping time of BBSEA (which is equal to the total number of box selections under BBSEA) is equal to $\sum_{m=1}^{M} U_m^{\prime}$. Part~3 of Theorem~\ref{thm:BBSEA_performance} follows from standard change-of-measure arguments for multi-armed bandits \cite[Lemma 1]{kaufmann2016complexity}. This completes the proof.



\end{document}